\documentclass[journal,12pt,onecolumn,draftclsnofoot,]{IEEEtran}

\usepackage{papercommands}
\usepackage{todonotes}
\usepackage{adjustbox}
\usepackage{diagbox}

\begin{document}

% \begin{frontmatter}

\title{An Efficient Drifters Deployment Strategy to Evaluate Water Current Velocity Fields}

% \author{Murad Tukan\corref{mycorrespondingauthor}}
% \cortext[mycorrespondingauthor]{Corresponding author}
% \address{Department of Marine Technologies, University of Haifa, 3498838 Haifa, Israel}
% \ead{mtukan@campus.haifa.com}
% \author{Eli Biton}
% \address{Dept. of physical oceanography, Israel Oceanographic and Limnological Research, 3109701 Haifa, Israel.}
% \ead{elib.ocean@gmail.com}
% \author{Roee Diamant}
% \address{Department of Marine Technologies, University of Haifa, 3498838 Haifa, Israel.}
% \ead{roee.d@univ.haifa.ac.il}

\author{Murad~Tukan,\  Eli~Biton, \ Roee~Diamant,~\IEEEmembership{Senior~Member,~IEEE}% <-this % stops a space
\thanks{M.~Tukan(corresponding author mtukan@campus.haifa.ac.il) and R.~Diamant are with the Department of Marine Technologies, University of Haifa, 3498838 Haifa, Israel.}% <-this % stops a space
\thanks{E.~Biton is with the Dept. of physical oceanography, Israel Oceanographic and Limnological Research, 3109701 Haifa, Israel.}% <-this % stops a space
\thanks{This work was sponsored in part by the MOST-BMBF German-Israeli Cooperation in Marine Sciences 2018-2020 (Grant \# 3-16573), by the MOST action for Agriculture, Environment, and Water for the year 2019 (Grant \# 3-16728), and by a grant from the University of Haifa’s Data Science Research Center. This work has been submitted to the IEEE for possible publication.
Copyright may be transferred without notice, after which this version may
no longer be accessible.}}

\maketitle
\begin{abstract}
Water current prediction is essential for understanding ecosystems, and to shed light on the role of the ocean in the global climate context. Solutions vary from physical modeling, and long-term observations, to short-term measurements. In this paper, we consider a common approach for water current prediction that uses Lagrangian floaters for water current prediction by interpolating the trajectory of the elements to reflect the velocity field. Here, an important aspect that has not been addressed before is where to initially deploy the drifting elements such that the acquired velocity field would efficiently represent the water current. To that end, we use a clustering approach that relies on a physical model of the velocity field. Our method segments the modeled map and determines the deployment locations as those that will lead the floaters to 'visit' the center of the different segments. This way, we validate that the area covered by the floaters will capture the in-homogeneously in the velocity field. Exploration over a dataset of velocity field maps that span over a year demonstrates the applicability of our approach, and shows a considerable improvement over the common approach of uniformly randomly choosing the initial deployment sites. Finally, our implementation code can be found in~\cite{opencode}.
\end{abstract}


\section{Introduction}
\label{sec:intro}

Knowledge and information about the ocean’s flow is highly applicable to scientific purposes such as climate change, global heat distribution, air-sea interactions, eddy formations, convection, tides, biological productivity, to name a few. Prediction of the water current is also required for operational needs such as marine conservation, search and rescue, the fishing industry, navigation, the development of marine infrastructure, tracking oil-spill distribution, tsunami warnings, and renewable energy. The study of the complex oceanic flow field variability - characterized by a wide range of spatial-temporal scales of processes - demands the combination of physical models and observations. In particular, the latter can be used to calibrate or to validate the model's parameters, as well as to serve as a database for statistical
evaluation.

Existing systems for directly measuring the water current (WC) mostly involve current profiling at a fixed deployment position, such as acoustic Doppler current profiles (e.g., ADCPs), or cover extensive areas, but only of the sea surface (e.g., HF radar and satellite elevation data). The data collected is used as an input for analytical and numerical models in a data assimilation fashion~\cite{kuznetsov2003method,miller2007numerical,santoki2011assimilation}. These models rely on local environmental information such as temperature, wind velocity, bathythermy and bathytermic, and need to be calibrated~\cite{castellari2001prediction,carrier2014impact} Another solution is to use Lagrangian floaters to in-situ evaluate the velocity field for data assimilation. Recently,~\cite{diamant2020prediction} has developed a methodology to estimate the 3D flow field, based on the tracking of the trajectories of surface or submerged floaters. The method has been successfully implemented to restore/complete data gaps in a flow field. 
Other works have shown that the accuracy of the estimated flow field depends not only on the number of floats, but also on the locations of their initial deployments~\cite{molcard2003assimilation}. 
As such, it is necessary to develop sophisticated methods for optimal dispersion of the floaters. Clearly, this optimal setup should be related to the flow’s characteristics. 

In this work, we develop a scheme how to plan ahead the initial positioning of Lagrangian floaters, such as to improve flow field reconstruction~\cite{diamant2020prediction, callaham2019robust}. 

Our scheme is useful as a perpetration for in-situ calibration of a given water current flow field model. We analyze a given flow field
map that is generated by a physical model to plan where to initially deploy a fixed number of floats. After executing our deployment planning scheme, the floats are released at the proposed locations and move freely with the water current for a fixed time frame, while their locations is tracked by e.g., acoustic positioning~\cite{jaffe2017swarm}. After this in-situ operation, the trajectories of the
floats are used to generate a flow field using e.g., \cite{diamant2020prediction} or to validate or calibrate the given physical model. In the above described process, as illustrated good coverage of the flow field by the floats is required. This is illustrated in the two examples in Fig.~\ref{fig:real-time}. Here, we include a map of the WC velocity field's magnitude and heading by the length and direction of the dark arrows. We observe homogeneous areas in the WC by small variations of the arrows and areas of a complex WC structure by a large diversity in the size and direction of the arrows. A group of 3 floats marked in orange lines are deployed in homogeneous sections of the flow field and thus do not capture the complexity of the water current, whereas a group of floats marked in green lines is well distributed to pass through all
diverse sections in the flow field map. Clearly, the difference between the two groups is by their initial position. Our approach applies machine learning tools over the given flow field map. Particularly, inspired by computational geometry, we turn to coresets for the task of planning the initial location of the floaters. Informally speaking, given input data, a coreset is a weighted subset of the original data that approximates the original data in some provable sense with respect to a (usually infinite) set of queries or models and an objective loss/cost function~\cite{tukan2020coresets}. We use coresets to find those \say{regions of interest} where the floaters should visit in order to most efficiently represent the WC structure. The method is tested on a finite resolution circulation model results that span over a year. Results show that using our algorithm, the floats better capture the complex structure of the WC, compared to the case of randomly deploying the floaters.

\begin{figure}[htb!]
    \centering
    \includegraphics[width=\linewidth]{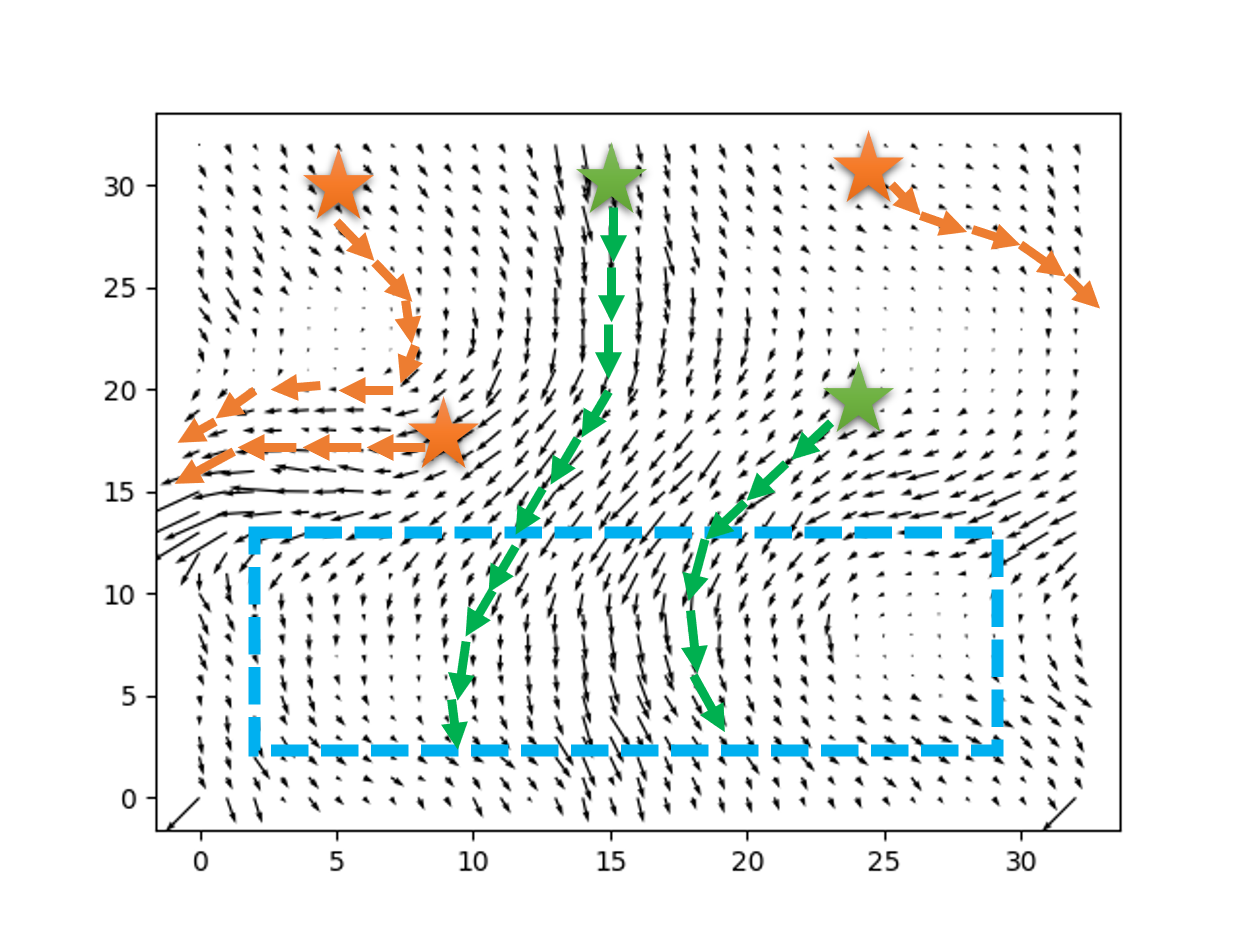}
    \caption{The setting of a real-time sea experiment. This map was generated using SHYFEM~\cite{umgiesser2004finite} (System of Hydrodynamic Finite Element Modules) by setting different variables, e.g., bathymetry, wind velocity, wind direction, etc. The direction of any arrow describes the direction of the WC at that discrete position $p$, while the length of the arrow represents the speed of the WC at $p$.}
    \label{fig:real-time}
\end{figure}

To the best of our knowledge, our approach is the first to consider the problem of optimal deployment of floats for the task of WC prediction and is the first to use coresets for oceanographic applications. Our contribution is threefold:
\begin{enumerate}[label=(\roman*)]
    \item A coreset-based solution for WC segmentation. A novel partitioning of the WC velocity field into segments of homogenous WC.
    \item  A clustering scheme for the segmentation of the WC's flow field. Primarily, reducing the problem of WC clustering into an instance of sets clustering. 
    \item A sub-optimal solution to determine the deployment location of floaters for WC's estimation. A graph theory-based approach to plan the deployment location of floaters such that the floaters explore the different clusters of the WC.
\end{enumerate}

The paper is organized as follows. Related work is discussed in Section~\ref{sec:related}. System setup and preliminaries are given in Section~\ref{sec:system}. In Section~\ref{sec:method}, we discuss the methodology of our proposed approach. Section~\ref{sec:exp} presents the numerical and experimental results, and conclusions are drawn in Section~\ref{sec:conclusions}.

\section{Related Work}
\label{sec:related}
Constructing WC flow fields based on Lagrangian particle trajectories has been gaining increasing attention over recent years~\cite{salman2008using,chapman2017can}. The main difference between the available approaches lies on the formalization of the relationships between the floaters' trajectories. In~\cite{bai2018motion}, a model for the flow field is used to find such a connection. Another solution is offered in~\cite{diamant2020prediction}, where the relations between drifter trajectories are calculated by statistical models. In~\cite{shi2021cooperative}, a cooperative solution is adopted to recover the flow field by formulating the integration error. Specifically, the motion-integration errors of multiple autonomous underwater vehicles (AUVs) in a 2D flow are obtained. The relation between the flow model and motion-integration errors is then formulated as a system of nonlinear equations followed by an iterative algorithm that is designed to estimate the flow field. While the above techniques for data assimilation are able to merge measurements with a model to estimate the WC, as shown in~\cite{rafiee2009kalman}, the results are sensitive to the initial deployment of the drifters. More specifically, an overly-close deployment would not capture the spatial dependency of the velocity field, while a too-far deployment, even below the Rossby radius of deformation, may break the assumed correlation between the sensors' drifting velocity.

A key challenge in determining the floaters' deployment locations is to spread the sensors across diverse sections of the explored area. One option is to allow maneuvering such that floaters can escape areas of homogeneous velocity field~\cite{hollinger2013sampling}. Another option is to direct the initial floater positions along the out-flowing branch of Lagrangian boundaries for better relative dispersion of floaters~\cite{molcard2006directed}. 
In contrast, in~\cite{hansen2018coverage} a sequential protocol was employed, where floaters are deployed one after the other, and their deployment locations are based on the trajectory obtained by the previously deployed floaters. First, the flow field is estimated using Gaussian processes (GP)~\cite{kim2011gaussian}, followed by trajectory estimation using OpenDrift~\cite{dagestad2018opendrift}. The estimated trajectories are then ranked to find the next deployment location with the aim of obtaining a longer, unexplored, trajectory. The process then repeats until all floaters are deployed. While this method holds potential for exploring non-homogeneous patches in the velocity field, its optimality can only be reached when a large number of floats are in use. Further, the method makes perhaps a too hard assumption that the WC is stable throughout a long enough observation window to deploy and recover the floats one by one.

While the models above are comparative to our work, they either assumed that (i) the moving agents have the ability to maneuver their own path, (ii) a large number of floaters is available to ensure good quality, (iii) or the velocity field is stable throughout a long enough observation window. These assumptions may be too hard in practical cases where the WC is time-varying, and when the number of floaters is limited. For these cases, we present an alternative solution.

\section{System model}
\label{sec:system}
\subsection{Setup details}
Consider a set of $K$ submerged floaters $\mathcal{X} = \br{1,\cdots,K}$, each of which drifts with the WC for a time frame of $T$ seconds. During their operations, the floaters' locations are known - either through a self-navigation process or using acoustic localization~\cite{tan2011survey}. The analysis is performed over a given time instance, $0 < t < T$, that can be configured according to the expected time it takes a float to cover the given area for exploration. In this time frame, each of the floaters in $\mathcal{X}$ measures the velocity of the WC. This can be done directly, using sensors like Doppler velocity loggers; indirectly, using the time-varying position of the floaters; or by simple periodic surfacing to obtain GPS fixes as performed for the Argo floaters~\cite{alexandri2021time}. Assuming for simplicity, that at time instance $t$ only one of the floaters measures the WC, we construct vector $\mathbf{p}(t) = \left[p_x(t), p_y(t), p_z(t), t\right]$ for the $x$ and $y$ UTM coordinates of the floater, its depth $p_z(t)$ in meters, and the observation’s time instance, $t$, respectively. Similarly, we obtain vector $\mathbf{v}(t) = \left[v_x(t), v_y(t), v_z(t)\right]$ representing the WC’s speed at the $x$, $y$, and $z$ directions, respectively. %XX The choice of coordinates is not so common. Why not simply using $v(t) = \left[v_x(t), v_y(t), v_z(t)\right]$ ?XX

We consider two scenarios: \begin{enumerate*}[label=\arabic*)]
    \item the floaters are recovered after $T$ seconds and the prediction of the WC’s velocity field is performed offline, and
    \item the prediction is performed online based on past WC velocity observations, in which case the operation must involve communicating between the floaters.
\end{enumerate*}
The first scenario mostly applies to the validation of a WC model, while the second can assist in the path planning of a submerged vessel. In this work, we are interested in determining the initial deployment position, $\mathbf{p}(0)$, of the floaters such that the WC is best predicted. To this end, we rely on our previously developed technique~\cite{diamant2020prediction} as a utility metric to evaluate the WC's velocity field from the floaters' trajectories.

\subsection{Assumptions}
We make the following assumptions. First, the floaters are assumed to be Lagrangian, such that their motion is completely attributed to the WC. We also assume the existence of a WC model that provides velocity predictions in a two-dimensional plane for the explored area. The spatial resolution of the given model is fixed, and the accuracy of our approach is directly related to the model's accuracy.

The interesting case that we are aiming for is a non-homogeneous WC with patches of homogeneity, such that a number of floaters are needed in order to well explore the WC's velocity field. These patches are assumed to change slowly in space, such that their borders are smooth and form convex sets. An example of such a velocity field is presented in Fig.~\ref{fig:toy_input} with arrows representing the magnitude and direction of the WC in a single cell in space. To simulate the floaters' motion within the modeled WC, the trajectories of the floaters can follow these arrows. The example in Fig.~\ref{fig:toy} shows such motion for a set of two floaters. We observe differences in the trajectory of the simulated floaters, which reflects the non-homogeneity of the WC.

%Throughout the paper, we assume the patches of homogeneous WC (similar direction and speed) are smooth and form convex sets, i.e., for every two points in a set of WC, the line connecting the pair of points is also contained in the set.

\subsection{Preliminaries}
\paragraph{Notations} For integers $n$ and $d \geq  2$, we denote by $[n]$ the set $\br{1, \cdots , n}$, by $\REAL^{n \times d}$ the union over every $n \times d$ real matrix, and by $I_d \in \REAL^{d \times d}$ the identity matrix. A matrix $A \in \REAL^{d \times d}$ is said to be  \begin{enumerate*}[label=(\roman*)]
    \item an orthogonal if and only if $A^T A = A A^T = I_d$, or
    \item a positive definite matrix if and only if for every column vector $x \neq 0_d$, $x^T A x > 0$.
\end{enumerate*}
For every set $A \subseteq \REAL^d$, we denote by $\abs{A}$ the number of elements of $A$. Finally, throughout the paper, vectors are addressed as column vectors.

\subsubsection{Volume approximation}
In what follows, we define what is known as the L\"{o}wner ellipsoid, a tool that will aid us in obtaining an $\eps$-coreset in the context of volume approximation.

\begin{definition}[Theorem III,~\cite{john2014extremum}]
Let $L \subseteq \REAL^d$ be a set of points. Let $c \in \REAL^d$ be a vector, and let $G \in \REAL^{d \times d}$ be a positive definite matrix. We say that ellipsoid $E = \br{x \in \REAL^d \middle| \term{x-c}^T G \term{x - c} \leq 1}$, is an \textit{MVEE} (short for the Minimum Volume Enclosing Ellipsoid) of $L$ if 
\begin{equation}
\frac{1}{d} \term{E - c} + c \subseteq \conv{L} \subseteq E,    
\end{equation}
where $E - c$ denotes the set $\br{x - c \middle| x \in E}$, $\frac{1}{d} E$ denotes the set $\br{\frac{1}{d} x \middle| x \in E}$, and $\conv{L}$ denotes the convex hull of $L$.
\end{definition}

\begin{definition}[Similar to that of \cite{todd2007khachiyan}]
\label{def:lownerEllip}
Let $\eps > 0$ be an approximation factor, and let $X \subseteq \REAL^{d}$ be a set of points. The set $S \subseteq X$ is defined to be an $\eps$-coreset for the \textit{MVEE} of $X$, if  
\begin{equation}
\vol{\mvee{X}} \leq \term{1 + \eps} \vol{\mvee{S}},    
\end{equation}
where $\vol{A}$ denotes the volume of $A$ and $\mvee{A}$ denotes the \textit{MVEE} of $A$, for any $A \subseteq \REAL^d$.
\end{definition}

\subsubsection{Clustering of sets}
To determine the initial location of the floaters, we first need to perform clustering to group together sets of similar WC. Thus, the following definitions will be used for the task of clustering.

\begin{definition}[Variant of Definition 2.3,\cite{jubran2020sets}]
\label{def:nmset}
Let $m,n,d$ be a triplet of positive integers. An $m$-set $P$ is a set of $m$ distinct points in $\REAL^d$. An $(m,n)$-set is a set $\mathcal{P} = \br{P \middle| P \subseteq \REAL^d, \abs{P} = m}$ such that $\abs{P} = n$.
\end{definition}

The following defines a distance between an $m$-set and a set of $k$ centers.

\begin{definition}[Variant of Definition 2.1,~\cite{jubran2020sets}]
Let $\term{\REAL^d, D}$ be a metric space, where $D : \mathbb{P}\term{\REAL^d} \times \mathbb{P}\term{\REAL^d} \to [0, \infty)$ be a function that maps every two subsets $P, C \subseteq \REAL^d$ to
\begin{equation}
D\term{P,C} = \min\limits_{p \in P, x \in C} \norm{p-x}_2^2.
\end{equation}
For an integer $k \geq 1$, define $X_k = \br{C \subseteq \REAL^d \middle| \abs{C} = k}$.
\end{definition}

The following defines a coreset for the \emph{sets clustering} problem~\cite{jubran2020sets}.

\begin{definition}[\cite{jubran2020sets}]
\label{def:coresetSets}
Let $n,m, k$ be a triplet of positive integers, $\mathcal{P}$ be an $(n,m)$-set as in Definition~\ref{def:nmset}, and let $\eps,\delta > 0$ denote the approximation error and probability of failure, respectively. $\term{\mathcal{S},v}$ is an $\eps$-coreset, where $\mathcal{S} \subseteq \mathcal{P}$ and $v : \mathcal{S} \to [0, \infty)$ is a weight function, if for every $C \subseteq \mathcal{X}_k$
\begin{equation}
\abs{\sum\limits_{P \in \mathcal{P}} D\term{P, C} - \sum\limits_{Q \in \mathcal{S}} v(Q) D\term{Q, C}} \leq \eps \sum\limits_{P \in \mathcal{P}} D\term{P, C},
\end{equation}
occurs with probability at least $1-\delta$.
\end{definition}

\subsubsection{Prediction of WC}
In our work, we use the scheme in \cite{diamant2020prediction} as cost function for predicting the WC. The prediction is based on calculating a function that links the positions and velocities of the floaters. This function can be linear, in which case the calculation is performed by a weighted least squares; or non-linear, in which case
the estimation involves support vector regression with a non-linear kernel function. Once the relation between the floaters' positions and their velocity is established, the WC's velocity at any given location (within the area explored by the floaters) is evaluated by operating the resulting function over the given location. %The accuracy of such a prediction serves as a utility function for setting the floaters' initial deployment position, $\mathbf{p}(0)$.

\section{Methodology}
\label{sec:method}

\begin{figure*}[htb!]
    \centering
    \begin{adjustbox}{max height=0.7\textwidth}
    \includegraphics[width=2\textwidth,height=1.5\textheight]{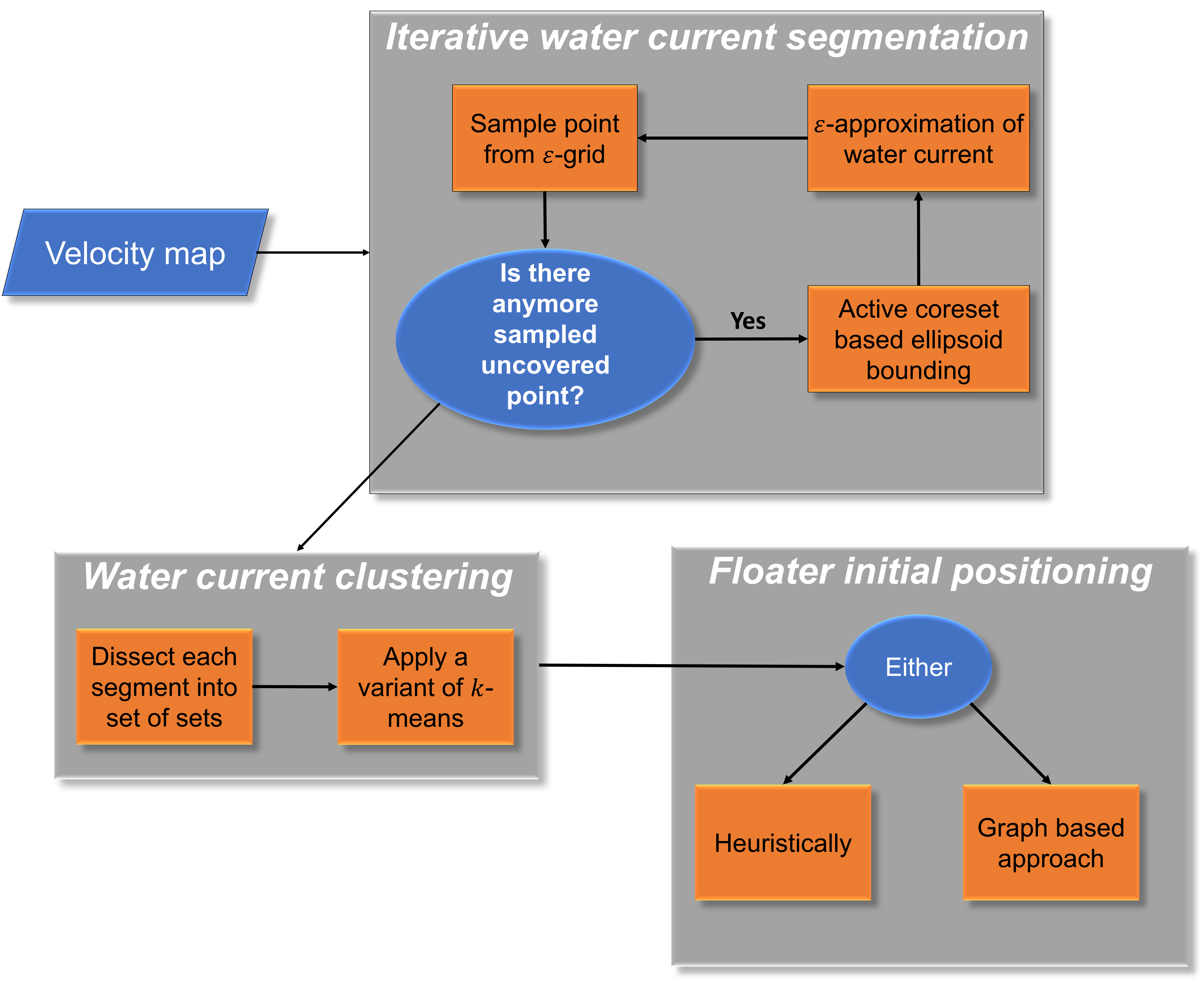}
    \end{adjustbox}
    \caption{A flow chart illustrating our approach for determining the floaters' best deployment position.}
    \label{fig:erd}
\end{figure*}

\subsection{Key idea}

Recall that we are interested in a solution that, given a WC flow field map, how to setup the deployment of a fixed number of Lagrangian floaters such as to best explore the flow field in-situ. The problem of setting the floaters' initial deployment is treated here as maximizing the information gained by the floaters with respect to the actual WC velocity's field. Such a problem can be reduced to the robot coverage path planning problem~\cite{7989583}, which aims to provide full coverage of an explored area while also minimizing the number of repeated visits. In the context of our problem, a variant of the coverage path planning problem is used: Given $K$ robots without the ability to control their movement, and a state space where each state moves the robot to a different state, the goal is to cover as many states as possible while moving through already discovered states as little as possible. This problem can be shown to be \emph{NP-hard}~\cite{zheng2005multi}. However, in our case, the space is not continuous, but rather discrete and bounded by the resolution of the given model. Such a setting simplifies the problem and makes it polynomial in nature (rather than exponential).

The steps of our algorithm are illustrated in the block diagram in Fig.~\ref{fig:erd}, and a toy example is illustrated in Fig.~\ref{fig:toy}. Given a map $\mathcal{M}$ of $M \times N$ velocity vectors forming a snippet of a WC's flow field (see example in Fig.~\ref{fig:toy_input}), the algorithm first partitions $\mathcal{M}$ into segments of homogeneous patches. This is translated into first applying an $\eps$-grid on the map. That is, dissecting the map into a set of $(M/\eps) \times (N/\eps)$ cells (see Fig.~\ref{fig:toy_grid}). Then, from each cell we sample one representative. For each sampled point, we next check whether the point is covered by a segment in which case we proceed to the next sampled point. Otherwise, we find the smallest ellipsoid in volume that encloses a homogeneous patch, including the sampled point. We then obtain an $\eps_\beta$-approximation towards the enclosed patch of the WC in the obtained ellipsoid from the previous step, as illustrated in Fig.~\ref{fig:toy_mapApprox}. The above steps are repeated over the set of sampled points until all points are covered.

As an approach for segmenting the flow field map, $\hat{M}$, a clustering scheme is applied. Each segment is dissected into an $(n,3)$-set (see Definition~\ref{def:nmset}), where $n = \ceil{\frac{\text{size of segment}}{3}}$, such that each point in the $(n,3)$-set is composed of its coordinates on the map $\mathcal{M}$ and the velocity vector that is present at those coordinates. We then normalize the velocity vector such that its norm is equal to the norm of its corresponding coordinates at $\mathcal{M}$. 
In the last stage of clustering, we generate a coreset for the \emph{sets clustering} problem on the merged set of sets $\overline{\mathcal{M}}$ (see Definition~\ref{def:coresetSets}), followed by a variant of the $k$-means algorithm, where $k$ here is equal to the number of floaters. The result is a set of $k$ centers that defines a clustering on $\widehat{\mathcal{M}}$ as shown in Fig.~\ref{fig:toy_clustering}. Finally, based on the clustered $\widehat{\mathcal{M}}$, we determine the deployment position of the $K$ floaters using two main techniques \begin{enumerate*}[label=(\roman*)]
    \item \textit{heuristics}: where the placement locations are chosen as the farthest point in the opposite direction of the dominating direction of each cluster or
    \item \textit{graph-based}: following the longest path formed in the flow filed map.
\end{enumerate*}

We handle the task of clustering by coresets. Coresets are a weighted subset of the input data. They were first introduced in computational geometry as a means to reduce the size of large datasets. Throughout recent years, coresets have been extended and developed for various optimization problems from different fields. One key component associated with coresets is that they aim to encapsulate the hidden structure in the data that the optimization problem at hand entails. Other approaches such as matrix sketches~\cite{woodruff2014sketching} or submodular maximization~\cite{badanidiyuru2014streaming} can also be used for clustering. The advantage of coresets is that it is a subset of the data where the coreset guarantee is satisfied for any query, e.g., any $k$ centers in the context of $k$-means clustering. Such coresets are referred as \say{strong coresets} in the literature~\cite{feldman2020core}.

\begin{figure*}[htb!]
    \centering
    \begin{subfigure}[t]{0.49\linewidth}
         \centering
         \includegraphics[width=\linewidth, height=.7\linewidth]{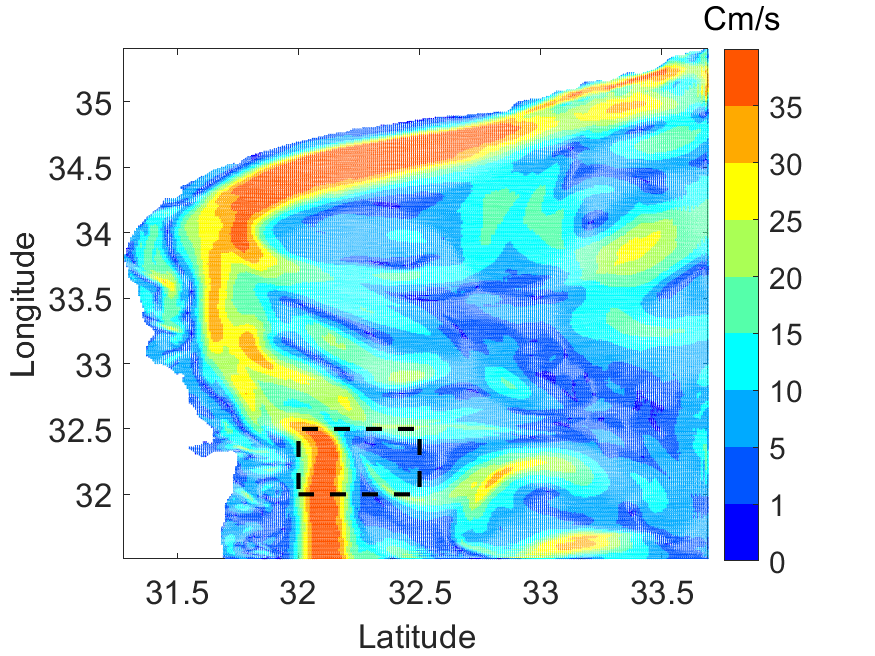}
         \caption{Map of velocities}
         \label{fig:toy_input}
     \end{subfigure}
     \begin{subfigure}[t]{0.49\linewidth}
         \centering
         \includegraphics[width=\linewidth,height=.7\linewidth]{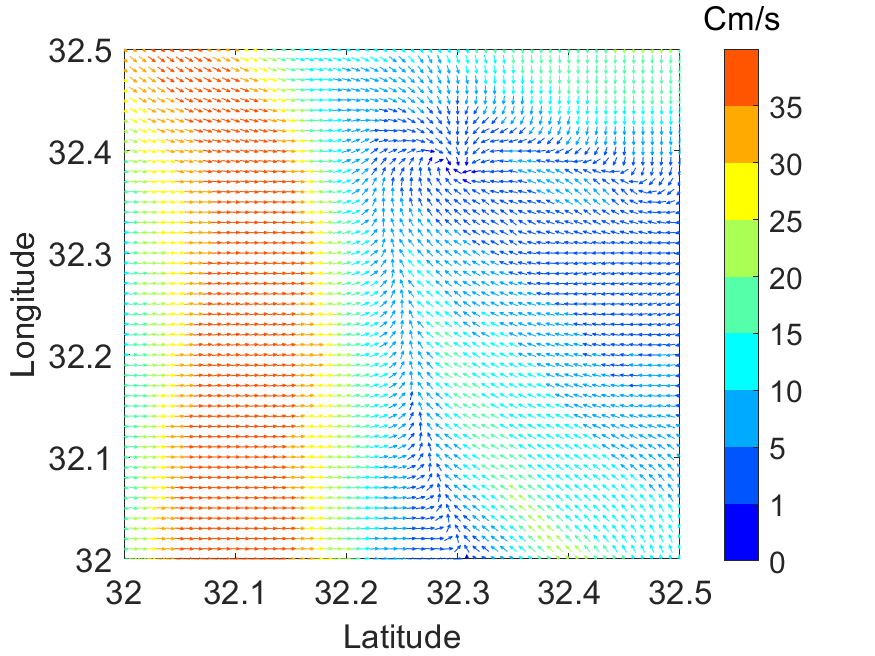}
         \caption{Zoomed area with respect to our toy map.}
         \label{fig:toy_input_zoom}
     \end{subfigure}
     \caption{A \textit{toy} example of a flow field map. The color bar denotes the magnitude of the velocity vectors. Example produced from the SELIPS model~\cite{goldman2015oil}. Fig.~\ref{fig:toy_input_zoom} depicts a zoomed area that is contained in the dotted black rectangle at Fig.~\ref{fig:toy_input}.}
     \label{fig:toy_example_map}
\end{figure*}

\begin{figure*}[htb!]
    \centering
    \begin{subfigure}[t]{0.49\linewidth}
         \centering
         \includegraphics[width=\linewidth,height=.6\linewidth]{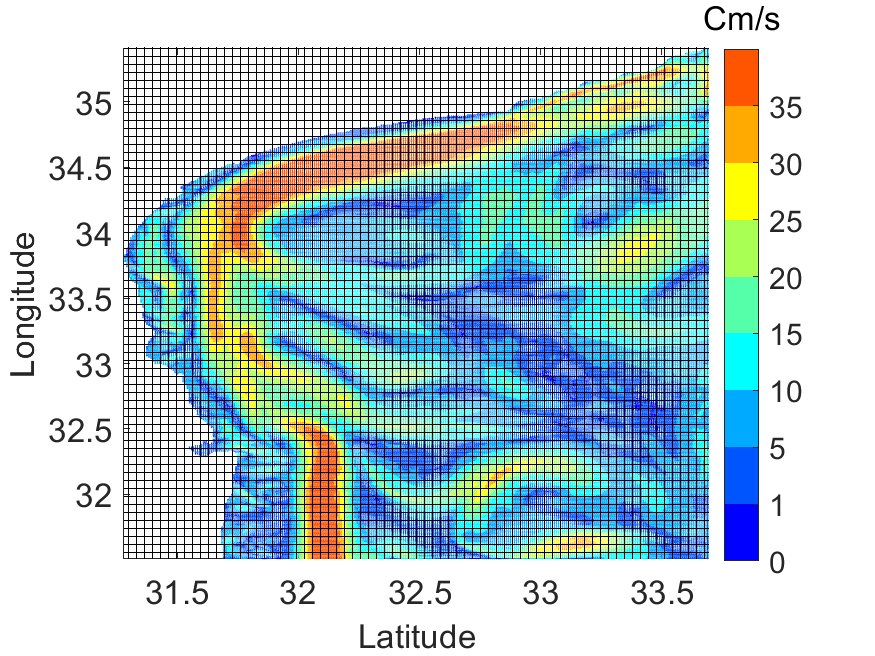}
         \caption{Map partitioning into grids}
         \label{fig:toy_grid}
     \end{subfigure}
     \begin{subfigure}[t]{0.49\linewidth}
         \centering
         \includegraphics[width=\linewidth,height=.62\linewidth]{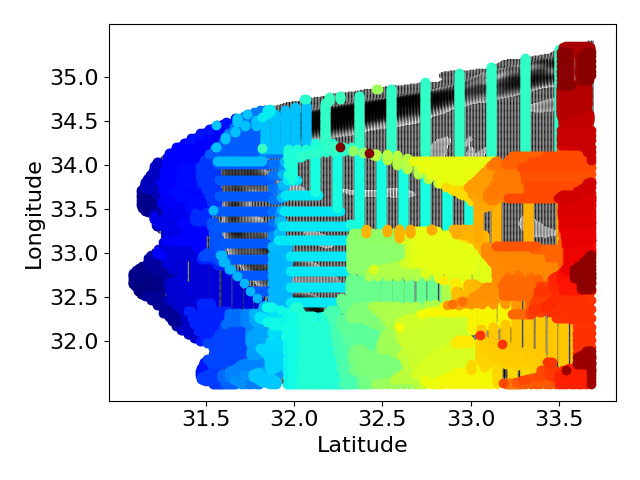}
         \caption{WC segmentation}
         \label{fig:toy_mapApprox}
     \end{subfigure}
     \begin{subfigure}[t]{0.49\linewidth}
         \centering
         \includegraphics[width=\linewidth,height=.62\linewidth]{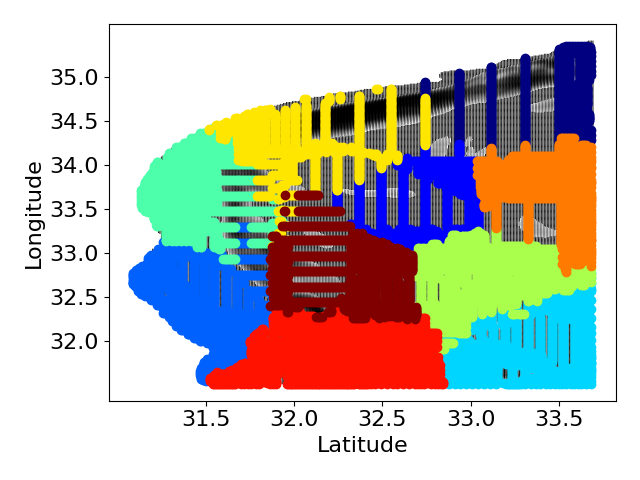}
         \caption{Clustering set of sets}
         \label{fig:toy_clustering}
     \end{subfigure}
    
    \caption{Illustration of running our model on the \textit{toy} example Fig.~\ref{fig:toy_example_map}. Fig.~\ref{fig:toy_grid} presents a partitioning of the flow field such that from each grid cell, a representative is randomly selected. Fig.~\ref{fig:toy_mapApprox} presents our segmentation which entails grouping areas of similar direction and speed. Fig.~\ref{fig:toy_clustering} clusters the segments to ensure that the number of clusters is equal to the number of floaters.}
    \label{fig:toy}
\end{figure*}

\subsection{Our Floater Deployment Scheme}
We formulate our problem as follows. Let $\mathcal{S}$ denote the space of all possible placements. For every $x \in \mathcal{X}$, let $f(x) \in \mathcal{S}$, denote the position of floater $x$. Assume that each $p \in \mathcal{S}$ is associated with a loss function $\phi : \mathcal{S} \to \mathcal{S}$ that maps $p$ to some state $q \in \mathcal{S}$. Finally, let $\mathcal{P}\term{\mathcal{S}}$ denote the power set of $\mathcal{S}$, $\pi : \mathcal{S} \to \mathcal{P}(\mathcal{S})$ denote the path of visited states given the initial state for a floater, $\ell(p(f(x))$ denote the length of the path associated with the floater $x$, and $\mathcal{S}_{i=1}^4$ be a set of orthants of $\mathcal{S}$ such that for each $i,j \in [4]$, $S_i \subseteq S$ and $S_i \cap S_j = \emptyset$ with $i \neq j$. The optimization problem is formalized by
\begin{equation}
\label{eq:optimization_prob}
\begin{split}
\max \quad& \sum\limits_{j=1}^4 \frac{\log{\term{\abs{\term{\bigcup\limits_{x \in \mathcal{X}} p\term{f(x)}} \cap \mathcal{S}_j}}}}{\sum\limits_{x \in \mathcal{X}} \ell(p(f(x))}\\
\mathrm{s.t.} \quad& x \in \mathcal{X}\\
\quad& f(x) \in \mathcal{S}
\end{split}
\end{equation}

In~\eqref{eq:optimization_prob}, the size of the union of different sets (denoted by the absolute function over a set) accounts for each state that is visited by some floater only once, and the loss function forces the solver to find initial states whose path must at least pass through one state from each of the subspaces $\br{\mathcal{S}_i}_{i=1}^4$ of $\mathcal{S}$. The loss function aims to guide the solver to choose placements that doesn't lead to infinite loops.

\subsubsection{Iterative WC segmentation}
Our solution starts by identifying segments in the WC's velocity field. Each segment contains a homogeneous set of WC vectors representing the WC's magnitude and direction. The task is performed by oracle-based algorithms. The data is assumed to be \say{hidden} and only available to the oracle, and the user is allowed to ask the oracle questions with \say{yes/no} responses. Such oracles are known by the term \textit{membership oracles}. The motivation behind such decisions is scalable algorithms for segmentation. Using the oracle-based approach, the emphasis is to segment the map into a set of segments using minimal oracle queries. In addition, such approaches enable the handling of large-scale velocity maps in near-linear time.

In our context, the oracle has the ability to distinguish between different current patches. With such an oracle, we find the minimal volume enclosing ellipsoid (or \textbf{MVEE} in short, see Definition~\ref{def:lownerEllip}) of each WC patch. Specifically, using the given membership oracle, a separation oracle can be constructed in polynomial time. The response of the separation oracle is \say{True} if a point lies inside the body of interest. If the point lies outside the body of interest, the oracle outputs a hyperplane, separating the point from the WC patch. Using the separation oracle, the ellipsoid method~\cite{grotschel1993ellipsoid} can be leveraged to find a $\term{1 + O\term{\eps}}$-approximation for the optimal \textbf{MVEE}. The time complexity of such algorithms is $\bigO{nd^4\log^{\bigO{1}}\term{\frac{d}{\eps}}}$; recall that in our setting, $d \in \bigO{1}$; hence, the time complexity of our algorithm is linear in the number of points $n$. We refer the reader to~\cite{grotschel2012geometric,lee2018efficient,9981428} for an extensive analysis of this method.

Once a WC patch $P$ has been enclosed by an ellipsoid, we proceed to obtain an $\eps_\beta$-approximation towards the volume of $P$, i.e., we aim to find $C \subseteq P$ such that
\begin{equation}
\frac{\vol{\conv{C}}}{\vol{\conv{P}}} \geq 1 - \eps_\beta.
\end{equation}

For this task, we first dissect $P$ to $\vol{P}\eps_\beta^{d}$ cells (see Definition~\ref{def:lownerEllip}). From each cell of this type, we uniformly choose a representative point at random. This ensures that the volume of the set of sampled points approximates the volume of $P$, which in turn approximates the structural properties of $P$~\cite{har2011geometric}.

\subsubsection{Clustering WC}
A fundamental clustering approach is $k$-means, which can also be used here to cluster WC. However, $k$-means will disregard the connectivity between points (segment points). Instead, we use sets clustering~\cite{jubran2020sets}, which is a generalization of $k$-means to cluster dependent sets of points.

Each approximated WC patch is partitioned into a set of triplets based on distance. More specifically, each point is associated with the closest two points to it based on Euclidean distance. The result is a $(3,n)$-set $\set{P}$, where $P\in\set{P}$ is a set of $3$ WC velocity vectors, and $n$ denotes the number of all such sets (see Definition~\ref{def:nmset}). The time complexity for finding a \say{sub-optimal} solution for such clustering is $\bigO{n\log{n} \term{nk}^{dk}}$~\cite{jubran2020sets}, where $n$ denotes the number of sets of points, $k$ denotes the number of desired clusters, and $d$ denotes the dimension of each point in the sets of points. Such a solution is, at most, worse than the optimal solution by a multiplicative factor of $\bigO{\log{n}}$. Leveraging the use of coresets, we can reduce the running time to $n\log{n}k^3 + \term{\frac{\log{n}}{\eps}dk^3}^{\bigO{dk}}$, while maintaining a solution that is associated with an approximation factor of $\bigO{\term{1+\eps}\log{n}}$~\cite{jubran2020sets}. It can be shown that solving the clustering problem on the coreset admits an approximation towards the optimal clustering obtained on all of the data, as follows.

\begin{claim}
Let $\set{P}$ be an $(n,3)$, $\eps \in (0,0.5)$, and $\term{\set{C},w}$ denote its $\eps$-coreset as in Definition~\ref{def:coresetSets}. Let $X_C$ denote the optimal clustering with respect to the coreset $\term{\set{C},w}$ and $X_P$ denote the optimal clustering with respect to $\set{P}$. Then
\[
\sum\limits_{P \in \set{P}} D\term{P,X_C} \in \term{1+O(\eps)}\sum\limits_{P \in \set{P}} D\term{P,X_P}.
\]
\end{claim}
\begin{proof}
Observe that
\begin{equation}
\begin{split}
\sum\limits_{P \in \set{P}} D\term{P,X_P} &\leq \sum\limits_{P \in \set{P}} D\term{P,X_C} \\
&\leq \frac{1}{1-\eps} \sum\limits_{P \in \set{C}} w(P) D\term{P,X_C}\\
&\leq \frac{1}{1-\eps} \sum\limits_{P \in \set{C}} w(P) D\term{P,X_P}\\
&\leq \frac{1+\eps}{1-\eps} \sum\limits_{P \in \set{P}} D\term{P,X_P},    
\end{split}
\end{equation}
where the first inequality holds by definition of $X_P$, the second and last inequality follows from Definition~\ref{def:coresetSets}, and the third inequality holds by definition of $X_C$.

The claim holds since $\frac{1+\eps}{1-\eps} \leq 1+4\eps$ due to the fact that $\eps \in (0,0.5)$.
\end{proof}

Since the input of our algorithm is a discrete map, usually represented via a grid, the input to the clustering must also incorporate the coordinates of each point in any WC patch in the given map. For such a task, to each point $p$ in each triplet $P$, we concatenate the corresponding coordinates in the map, resulting in $\hat{p}$, while the set $P$ is then referred to as $\hat{P}$. To ensure fairness across the two feature vectors that $\hat{p}$ is composed of, we ensure that their norm is roughly equal through scaling. The set of all $\hat{P}$ is then passed to the sets clustering scheme, and the result is treated at the clustering on the original set $\mathcal{P}$.

\subsubsection{Towards optimal floater deployment while seeking maximal coverage}
Once the WC map has been clustered, we determine the deployment position of the floaters to obtain the best WC velocity field estimation.
We consider two types of solutions to the deployment strategy. The first is a heuristic approach, referred to as \textit{heuristic}, where each cluster is assigned a unique floater. The floater's location is then set at the farthest point along the negative of the dominating direction from the cluster, thereby obtaining the longest traversal inside the cluster.
Here, the dominating direction of a cluster refers to the direction that most points in the cluster either point to, or are very close to in terms of cosine similarity. This approach ensures that each cluster is covered, while allowing for additional data collection from the deployment position to the cluster's boundaries. However, the scheme is not optimal for the probable case where the number of clusters is lower than or equal to the number of floaters.

A more rigorous approach would be to employ concepts from graph theory, and we refer to it as \textit{graph-based} approach. Each cluster $S$ is represented as a directed graph $\mathcal{G}_S := \term{V_S, E_S}$, where each point in $S$ is assigned a vertex in $\mathcal{G}_S$. As for the set of edges of $\mathcal{G}_S$, an edge $e := v \to u$ exists in $\mathcal{G}_S$ if $q$ is reachable from $p$ following the velocity vector of $p$, where $p$ and $q$ are the corresponding points in $S$ of that of $v$ and $u$, respectively. In other words, an edge exists if, and only if, \begin{enumerate*}[label=(\roman*)]
    \item one can move from $p$ to $q$ using the velocity vector that is associated with $p$, and 
    \item $q \in B\term{p, 1}$ where $B(x,r)$ denotes a ball centered at $x$, with a radius of $r$.
\end{enumerate*}
At this stage, we have $K$ disconnected graphs for $K$ floaters. For each graph, we compute the longest path from the set of shortest paths between each pair of graph nodes by applying a breadth-first search (BFS) algorithm~\cite{bundy1984breadth} each time from a different node. The running time of this procedure is $O\term{\abs{V_S}^2 + \abs{V_S}\abs{E_S}}$ for any graph $\mathcal{G}_S$. Our above algorithm can be generalized by taking into account weights, in which case, the BFS algorithm is replaced by Johnson's algorithm~\cite{johnson1977efficient}. A modification of this graph-based approach, referred to as the \textit{inter-graph} scheme, connects these graphs by checking whether the roots and leaves of one graph can be connected to another graph. The connectivity is applied to each pair of graphs, and the resulting graph is denoted by $\mathcal{G}_{all}$. The BFS algorithm is then used again to compute the $K$ largest non-intersecting paths in $\mathcal{G}_{all}$. Finally, the floaters' deployment positions are set to be the starting vertices of the selected paths. 

\section{Experimental Analysis}
\label{sec:exp}

In this section, we evaluate the performance of our three strategies for determining the floaters' deployment positions, namely, heuristics, graph-based and inter graph-based. Without alternative benchmark for determining the initial position of the
floats for WC prediction, we compare the performance of our schemes to the common approach of sampling the initial deployment locations uniformly at random. We follow~\cite{diamant2020prediction} to evaluate the performance of the different deployment schemes in terms of the velocity field prediction. For any location $p(x,y)$ in the velocity field, we denote the ground truth velocity vector at $p(x,y)$ by $\begin{pmatrix} v_x\\ v_y\end{pmatrix}$, and the predicted velocity vector at $p(x,y)$ by $\begin{pmatrix} v_x^{\textit{pred}}\\ v_y^{\textit{pred}}\end{pmatrix}$. The latter is obtained by applying the method in \cite{diamant2020prediction}, each time for the different the floaters' trajectories as obtained after
deployment based on the four different deployment strategies. The prediction error is defined by 
\begin{equation}
\label{eq:velocity_error}
\rho_{speed} = \sqrt{\term{v_x - v_x^{\textit{pred}}}^2 + \term{v_y - v_y^{\textit{pred}}}^2}.
\end{equation}

Note that the method in \cite{diamant2020prediction} interpolates the floaters information towards the prediction of the flow field away from the floaters'
trajectories. As such, the prediction error in \eqref{eq:velocity_error} is calculated for each location in the explored area.

\subsection{Experimental Settings}
The method in~\cite{diamant2020prediction} offers a linear and a non-linear prediction models. Here, since we aim for complex flow fields with non-homogeneous
sections, we choose the latter that is based on support vector regression (SVR) model with a radial bases function (RBF) kernel. To train the \emph{RBF-SVR} model, we used a grid-search approach with cross validation~\cite{refaeilzadeh2009cross,claesen2015hyperparameter} to determine the best model parameters. The tuned parameters are \begin{enumerate*}[label=(\roman*)]
    \item $C$ -- a regularization parameters,
    \item $\epsilon$ -- an optimization-related parameter with respect to the SVR model, and
    \item $\gamma$ -- the exponent which controls the deviation of the spread of the radial basis function.
\end{enumerate*} 
For more details, we refer the reader to ``Scikit-Learn''~\cite{scikit-learn}.

As a WC model (and ground truth), we use $48$ WC maps, as produced by the \emph{SELIPS} model~\cite{goldman2015oil} for the Gulf of Haifa, Israel. The maps span over a time period of $12$ months. Model \emph{SELIPS} is an operational forecasting system based on \emph{POM}, a 3D numerical model for the simulation of ocean dynamics, with a horizontal resolution of about $1$ km. The output was given as $3$ hour averages of the velocity components. We explore the results for two options: 1) \textit{clean map}: the WC model is the same as the velocity field used to simulate the drift of the floaters, and 2) \textit{noisy map}: the velocity field used for the simulation is a noisy version of the WC model.
We explore the results for different number of floaters, $K$, and for different model parameters.

\subsection{Experimental Analysis}
\subsubsection{The Toy Example}
We start by showing the performance of each of our proposed deployment strategies on the \textit{toy} example in Fig.~\ref{fig:toy_input}. In Fig.~\ref{fig:toy_result}, we present the empirical cumulative distribution function (CDF) results of the velocity error vector~\eqref{eq:velocity_error}, as generated by predicting the flow field for each point in the map. We observe that the performance of our proposed strategies exceeds that of the uniformly choosing deployment strategy. This indicates that, from a statistical point of view, our method forms better candidates for deployment position strategies than sampling uniformly. Comparing the performance of our three schemes, we conclude that the inter-graph-based approach is better, mostly because of the complexity in the structure of the velocity field, which induces diversity in the WC. The inter-graph approach, which is more rigorous, can better capture this diversity.

\begin{figure}[!htb]
    \centering
    \includegraphics[width=.6\linewidth]{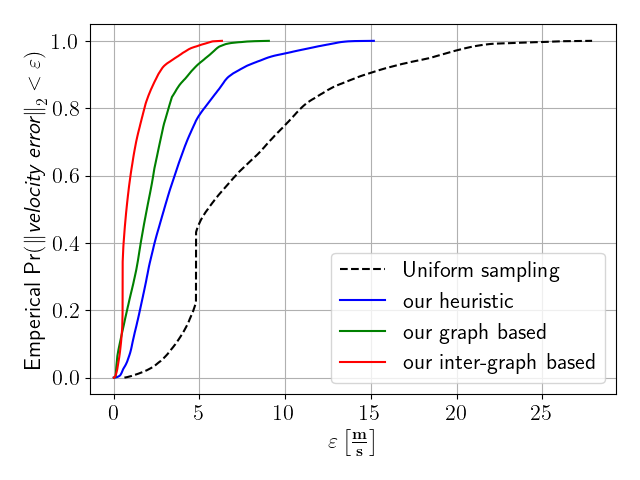}
    \caption{CDF of the norm of velocity error~\eqref{eq:velocity_error}. The results show our advantage upon using uniform sampling for determining deployment positions.}
    \label{fig:toy_result}
\end{figure}

\subsubsection{Choosing the \say{best} parameters for our model}
Our model relies on a predefined set of parameters. Specifically, the approximation error $\eps_\beta$ with respect to the volume of the explored area, which is required for the iterative segmentation stage, and the coreset size $\mu$ with respect to the clustering phase. To explore the sensitivity of the graph-based and inter-graph-based approaches to the choice of these parameters, we use the clean simulation setup for all $48$ WC maps to show that our model is robust against changes in these parameters. As seen in Fig~\ref{fig:params_exp}, best results are associated with $\mu := 1500$ and $\varepsilon_\beta := 0.05$. Still, the difference is rather small, which reflects on the robustness of our approach. In the below we choose $\mu := 1500$ and $\eps_\beta := 0.05$ for both the \textit{clean} maps, for the \textit{noisy} maps. 

\begin{figure*}[!htb]
    \centering
    \begin{subfigure}[t]{0.49\linewidth}
         \centering
         \includegraphics[width=\linewidth, height=.7\linewidth]{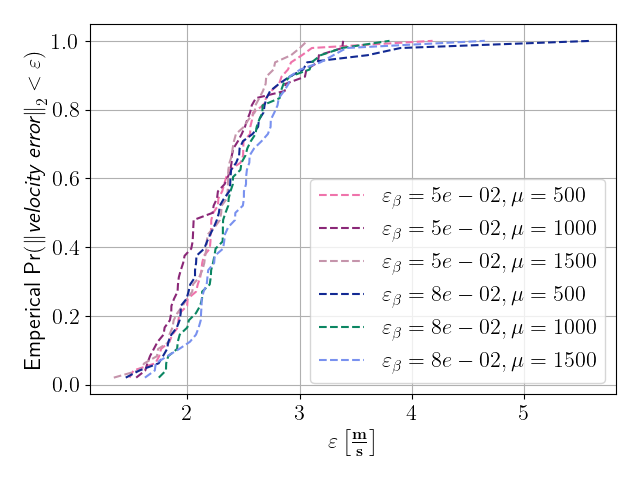}
         \caption{}
         \label{fig:graph_based_param}
     \end{subfigure}
     \begin{subfigure}[t]{0.49\linewidth}
         \centering
         \includegraphics[width=\linewidth,height=.7\linewidth]{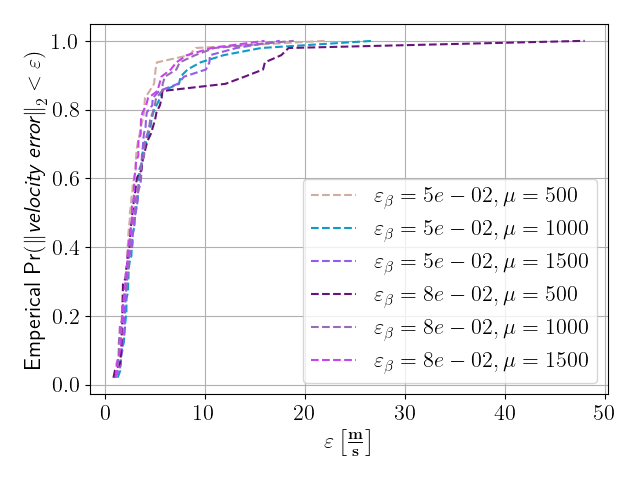}
         \caption{}
         \label{fig:inter_graph_based_param}
     \end{subfigure}
    
    \caption{CDF of the averaged prediction error across $48$ \say{clean} maps when using different model parameters with respect to our graph-based approach (left-most), and with respect to our inter-graph based (right-most).}
    \label{fig:params_exp}
\end{figure*}

\subsubsection{Comparison Against Uniform Sampling on a Variety of Maps}
\begin{figure}[htb!]
    \centering
    
   \includegraphics[width=.6\linewidth]{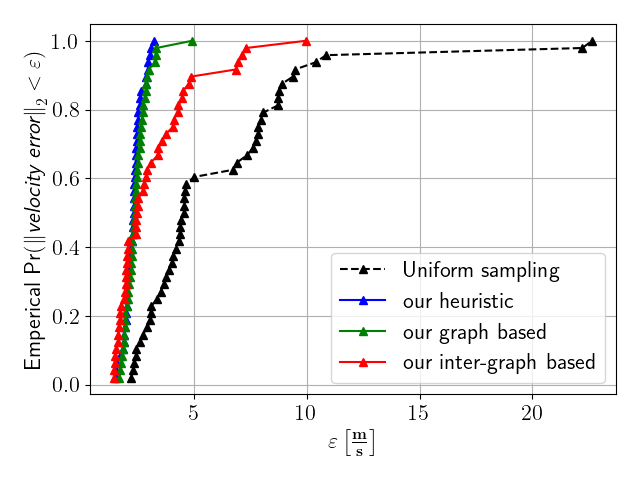}
    \caption{CDF of the averaged prediction error across $48$ maps.}
    \label{fig:across_48}
\end{figure}

We now explore the efficacy of the three schemes for different WC maps in the clean map setup against Uniform sampling. Fig.~\ref{fig:across_48} presents the CDF of the averaged prediction errors across the $48$ maps. We observe that the heuristic-based strategy for deployment outperforms most of the deployment strategies. This is due to the fact that there are no obstacles in each of the $48$ maps. That is an area in the WC flow field with zeroed-out velocity vectors, e.g., an island.
Observe that while the graph-based approach is comparable to the heuristic-based approach, at some point it starts being less efficient. This is due to the fact that the graph-based method needs denser clusters, i.e., the parameter $\eps_\beta$ needs to be smaller leading to larger and denser clusters, e.g., the effect of $\eps_\beta$ is best observed visually in Fig.~\ref{fig:toy_clustering}. In turn, the graphs generated from the clusters hold more information regarding the flow field. This will yield better results than our heuristic-based approach. In addition, the same behavior appears also when observing the inter-graph-based approach.

\subsubsection{Robustness Against Noise}

We explore the performance of the three schemes when the noisy map setup is considered.  
For our toy map $M$, we produce its noisy map $M^\prime$ by adding a Gaussian noise with zero mean, and a standard deviation equal to $\sigma\%$ of the standard deviation of $M$. The added noise is added to a fraction of the WC map, denoted by $\eta \in (0, 1)$, representing the corruption ratio of the WC map.

To generate a difference between the WC map and the velocity field used, we determine the deployment positions of the floaters, $\mathbf{p}(0)$, based on the given WC model (i.e., without the added noise), but calculate the trajectories of the floaters based on the noisy WC map. As a result, the assumed map is mismatched with the noisy one. The effect of $\eta$ and $\sigma$ on our toy example are presented in Table~\ref{tab:noiseEffect}. With a small corruption percent (small $\eta$), we observe that the resulted map is not that different from its clean version; see Figure~\ref{fig:toy_input}. As $\eta$ and $\sigma$ increase, the map loses its underlying structure almost entirely, as depicted at the rightmost lower cell of Table~\ref{tab:noiseEffect}.

\begin{table*}[t!]
    \centering
    \adjustbox{max width=\textwidth}{
    \begin{tabular}{|c||c|c|}
    \hline
    \backslashbox{$\eta\%$}{$\sigma\%$} & $15$ & $133$   
    \\ \hline
    $15$ & \includegraphics[trim=0 0 0 -20, scale = 0.4]{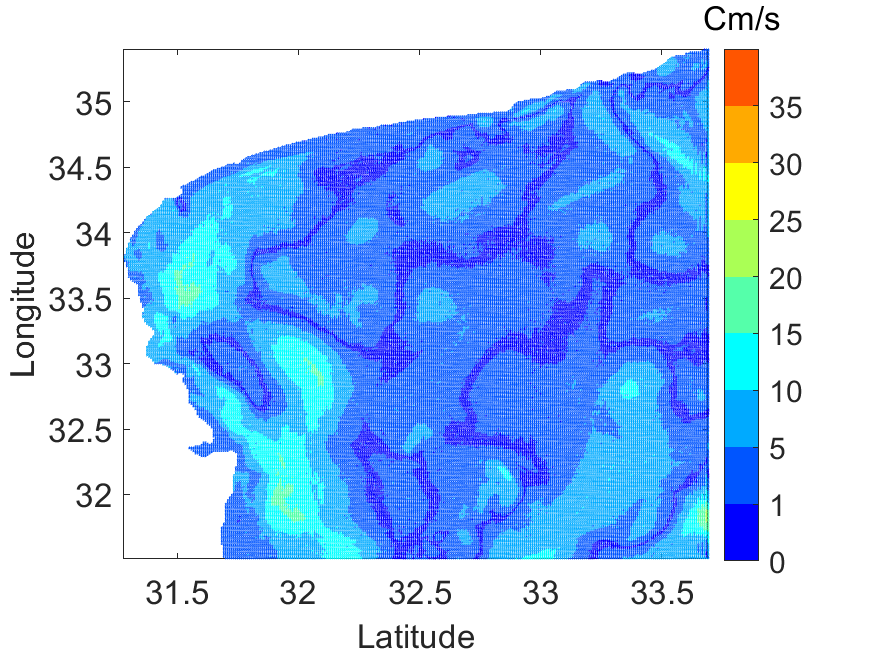} & \includegraphics[trim=0 0 0 -20, scale = 0.4]{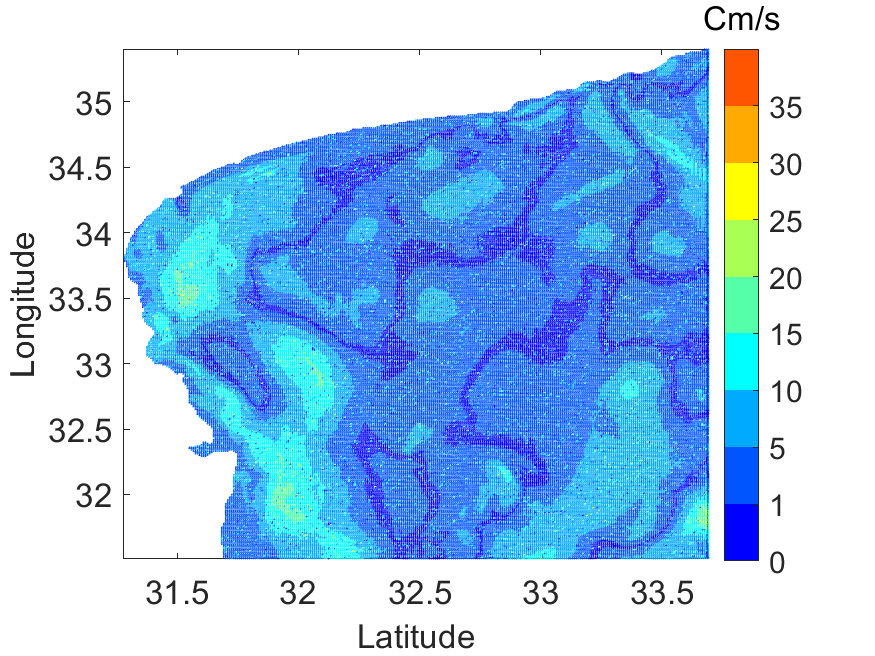}
    \\\hline
    $100$ & \includegraphics[trim=0 0 0 -20, scale = 0.4]{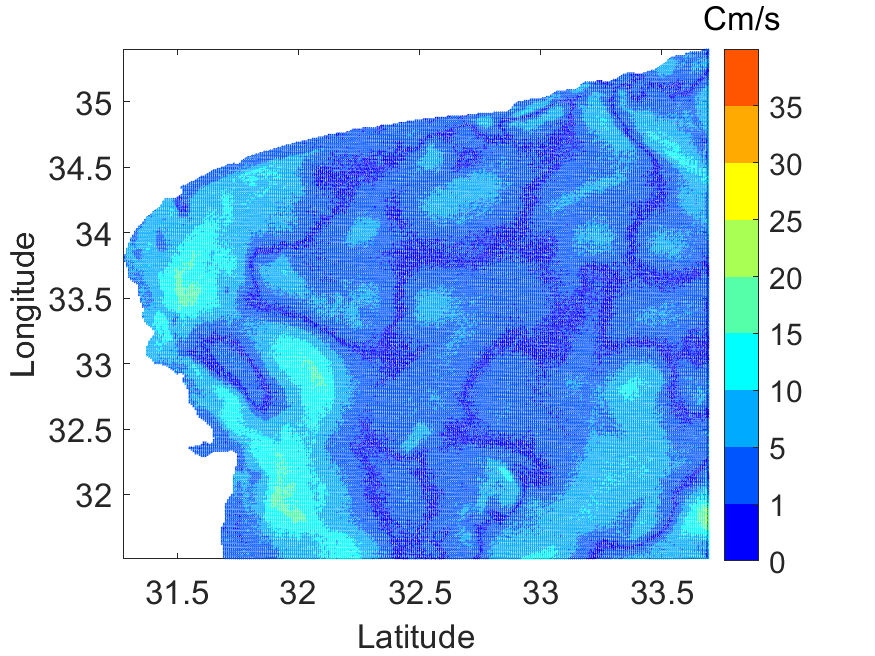} & \includegraphics[trim=0 0 0 -20, scale = 0.4]{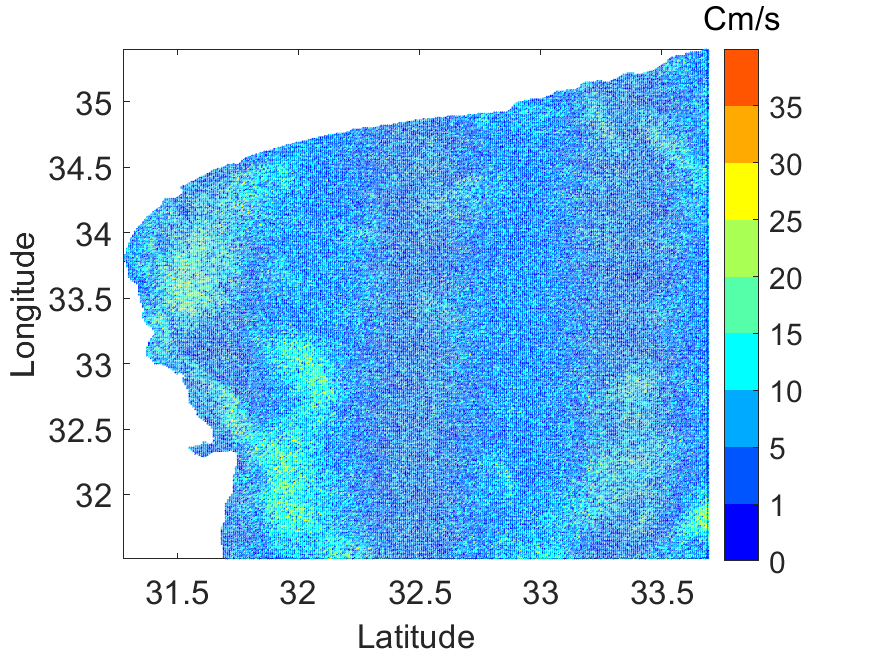}
    \\\hline
    \end{tabular}}
    \caption{The effect of added noise on our toy example.}
    \label{tab:noiseEffect}
\end{table*}

\begin{figure*}[htb!]
    \centering
    \begin{subfigure}[t]{0.325\linewidth}
    \includegraphics[width=\linewidth]{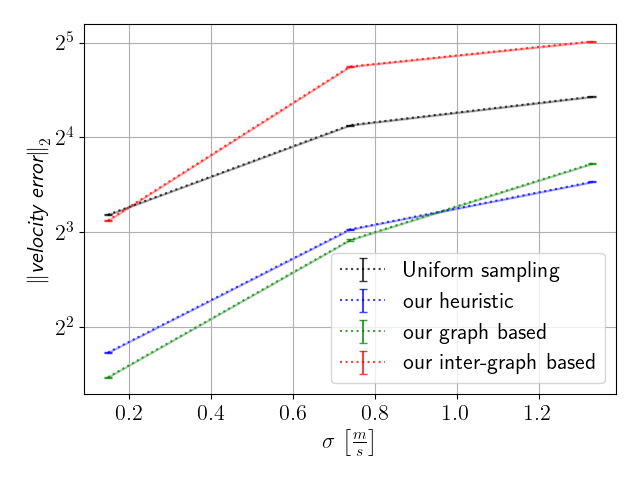}
    \caption{$\eta = 15\%$}
    \label{fig:15Percent}
    \end{subfigure}
    \begin{subfigure}[t]{0.325\linewidth}
    \includegraphics[width=\linewidth]{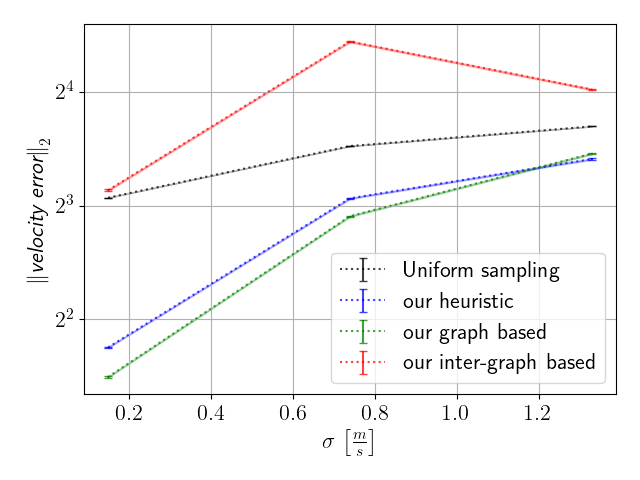}
    \caption{$\eta = 40\%$}
    \label{fig:40Percent}
    \end{subfigure}
    \begin{subfigure}[t]{0.325\linewidth}
    \includegraphics[width=\linewidth]{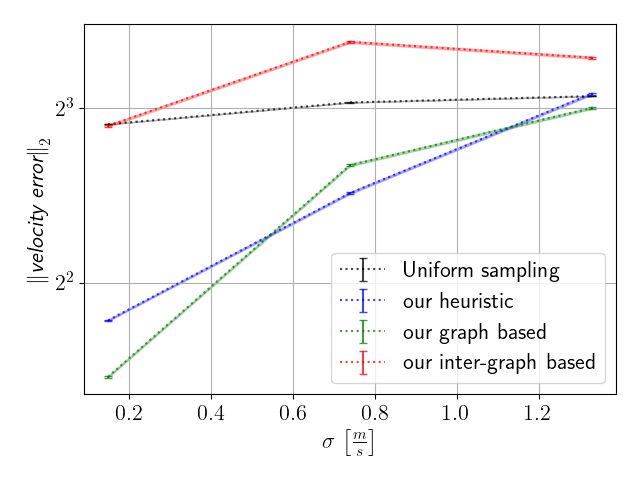}
    \caption{$\eta = 95\%$}
    \label{fig:95Percent}
    \end{subfigure}
    \caption{The averaged norm of the velocity error on our toy map as a function $\sigma$, when using three different corruption percents $\eta \in \br{0.15, 0.4, 0.95}$.}
    \label{fig:NoiseEXP}
\end{figure*}

\begin{figure*}[t!]
    \centering
    \begin{subfigure}[t]{0.49\linewidth}
         \centering
         \includegraphics[width=\linewidth, height=.7\linewidth]{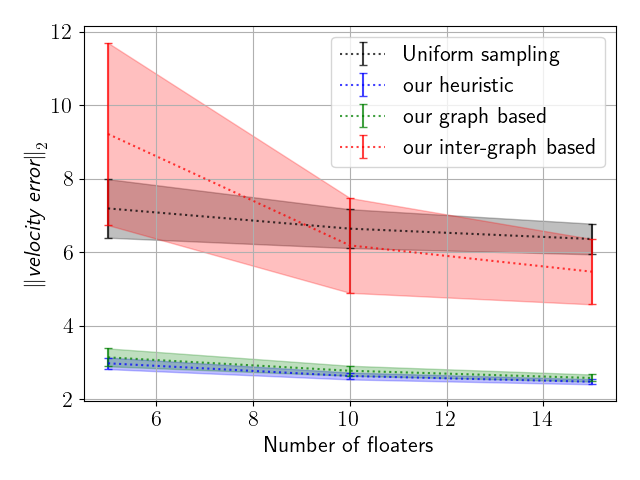}
         \caption{}
         \label{fig:K}
     \end{subfigure}
     \begin{subfigure}[t]{0.49\linewidth}
         \centering
         \includegraphics[width=\linewidth,height=.7\linewidth]{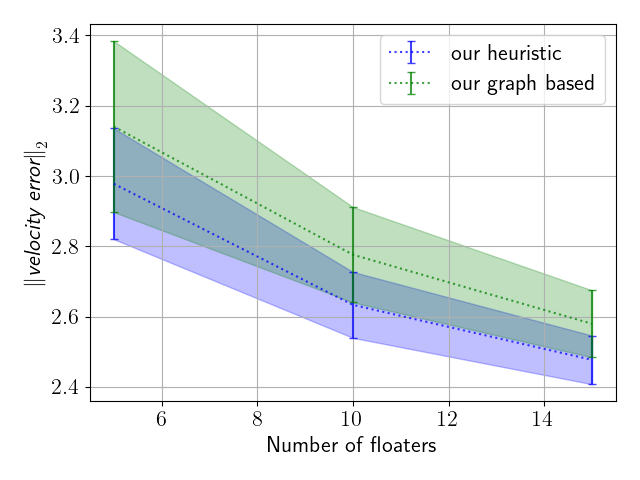}
         \caption{}
         \label{fig:K_zoom}
     \end{subfigure}
    
    \caption{On the left graph, the averaged norm of the velocity error across the $48$ maps as a function of the number of floaters $K$. On the right graph, we zoom in Fig.~\ref{fig:K} showing the averaged norm of the velocity error across the $48$ maps as a function of the number of floaters $K$. In both figures, the shaded regions denote a $95\%$ confidence bar.}
    \label{fig:KRes}
\end{figure*}

\begin{figure*}[t!]
    \centering
    \begin{subfigure}[t]{0.49\linewidth}
         \centering
         \includegraphics[width=\linewidth, height=.7\linewidth]{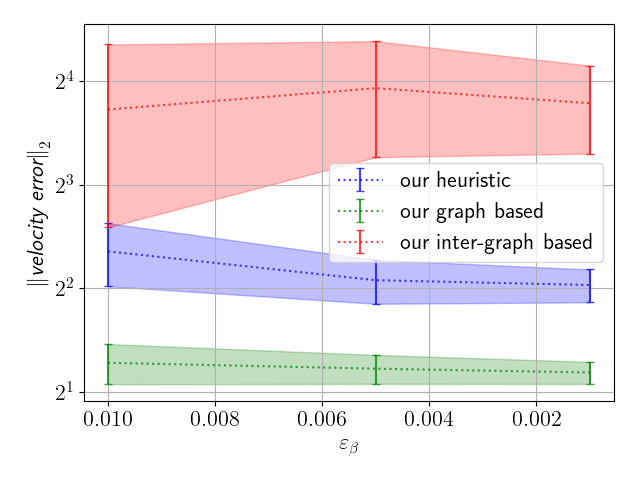}
         \caption{}
         \label{fig:EPSB}
     \end{subfigure}
     \begin{subfigure}[t]{0.49\linewidth}
         \centering
         \includegraphics[width=\linewidth, height=.7\linewidth]{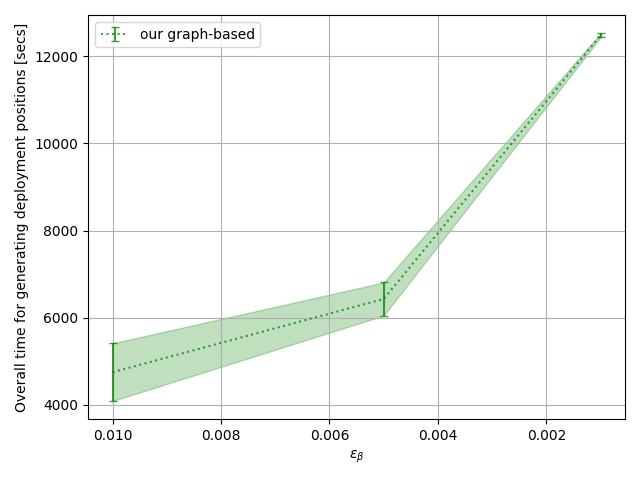}
         \caption{}
         \label{fig:EPSBTime}
     \end{subfigure}
     \caption{On the right, we present the average error of each of our deployment methods with respect to $WC$ reconstruction as a function of $\eps_\beta$. On the left, we present the running time needed to generate the positions for our graph-based approach $\eps_\beta$. Shaded regions denote the standard deviation with respect to the y-axis.}
\end{figure*}

The results are given in Fig.~\ref{fig:NoiseEXP}. We observe that the average error increases with $\sigma\%$. We argue that, as $\sigma\%$ increases, the correlation between the generated deployment positions on the clean map and the resulted noisy map becomes less strong leading to an increase in the average error. The heuristic-based and graph-based approaches still outperform the uniform sampling, but for the inter-graph-based approach which is sensitive to the smoothness of the maps, becomes less efficient in the presence of noise.

\subsubsection{Assessing the effect of $K$} In this experiment, we explore the effect of increasing the number of floaters on the error of reconstructing flow fields. Fig.~\ref{fig:KRes} presents the average error across the $48$ WC maps. We observe that as the number of floaters increases, the average error for each of the $4$ deployment strategies decreases. This is due to the fact that the amount of collective data also increases as more floaters are available, hinging upon a larger discovery of the underlying structure of the flow fields. The results show that graph-based and heuristic-based outperform uniform sampling by at least $225\%$. This is due to the nature of our approaches, which rely on information about the structural properties of the WC maps. On the other hand, the performance of our inter-graph-based approach is sometimes weaker than uniform sampling. This is mainly due to the fact that the former method requires denser graphs, ultimately leading to lower $\eps_\beta$.

\subsubsection{When to use each of our deployment strategies}
Finally, we explore the best setups that fit best each of our deployment strategies.

\paragraph{When accuracy matters more than run-time} Figure~\ref{fig:EPSB} presents the effect of $\eps_\beta$ on each of our proposed strategies. When $\eps_\beta$ decreases, the best strategy is the graph-based approach. It lines well with the observation that this method uses the underline structure of the map. However, the cost of using such low $\eps_\beta$ is reflected in Figure~\ref{fig:EPSBTime}. Using an AMD Ryzen Threadripper $3990$X $2.9$ GHz $64$-Core with $128$GB RAM, the run time increases from minutes to hours. The run time for the heuristic-based approach ranges between $3000$ seconds (at $\eps_\beta=0.01$) and $4500$ seconds (at $\eps_\beta=0.001$), while the run time of the inter-graph-based approach is similar to that of the graph-based approach.

\paragraph{Corrupted maps} We explore the performance of the three schemes when the given model is different than the real channel, i.e., using the noisy map setup. %As we expect, the WC model to be accurate only statistically, we argue that the noisy setup is the practical case. 
For each map $M$ from our set of $48$ flow field maps, we produce its noisy map $M^\prime$ by adding a Gaussian noise with zero mean, and a standard deviation equals to $\sigma\%$ of the standard deviation of $M$. Here the corruption percentage, $\eta$ is $100\%$. The results are given in Fig.~\ref{fig:noiseCorruption100} for different $\sigma^2$ values. We observe that the average error decreases as the added noise increases. This rather non-intuitive result is due to the fact that, as noise increases, the obtained map becomes similar to a Gaussian distributed. Consequently, the distribution of the map's entries can be better estimated from the learning phase, i.e., the path traversed by the floaters. That is, the impact of the floater's initial location becomes less dominant as the noise increases and the mismatch between the model and the actual map increases. That said, since the structure of the noise field still dominates over the added noise, we observe that our inter-graph-based approach is on par with the uniform deployment approach. This is because the former is the least information-collective approach among our three deployment strategies.

\begin{figure}[t!]
    \centering
    \includegraphics[width=.6\linewidth]{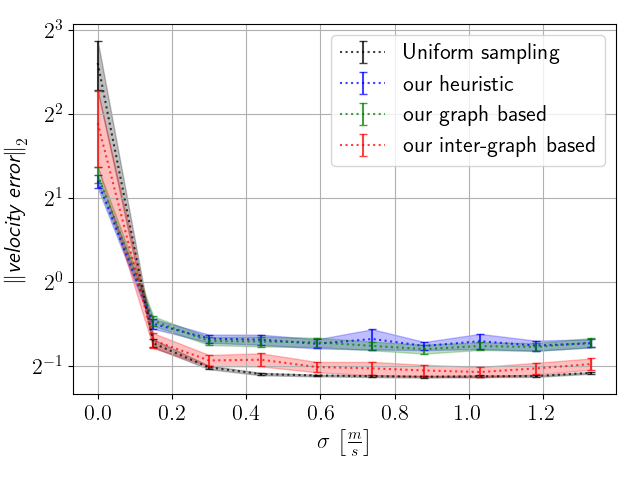}
    \caption{The averaged norm of the velocity error across the $48$ maps as a function $\sigma$, where the corruption percent $\eta$ is $100\%$. The shaded regions denote a $95\%$ confidence bar.}
    \label{fig:noiseCorruption100}
\end{figure}

\section{Conclusions and future work}
\label{sec:conclusions}

In this paper, we explored how to determine the initial deployment positions of a group of floaters to best evaluate the WC flow field.  Our approach relies on clustering a given model of the WC into segments, each of which is represented by a coreset, and determining the floaters' initial deployment positions with the aim of visiting all coresets under constraints: the number of floaters, and the time frame used for evaluation. 
We analyzed the results of our scheme over a database of $48$ WC maps that span over a year of measurements in the Gulf of Haifa, Israel. Compared to the uniform sampling benchmark, the results show that our scheme is more accurate in terms of the WC's prediction, and is more robust to mismatches between the given WC model and the actual one. Future work will identify gaps in the given model and complete them by guiding the floaters to visit these locations.

\section{Acknowledgements}
This work was supported in part by the MOST action for Agriculture, Environment, and Water for the 490 year 2019 (Grant $\#$ 3-16728) and by the the University of Haifa’s Data Science Research Center.

\bibliographystyle{IEEEtran}
\typeout{}
\bibliography{IEEEabrv,refs}

\end{document}